\documentclass{article}

\usepackage[preprint]{neurips_2026}

\usepackage[utf8]{inputenc}
\usepackage[T1]{fontenc}
\usepackage{hyperref}
\usepackage{url}
\usepackage{booktabs}
\usepackage{amsfonts}
\usepackage{amsmath}
\usepackage{amssymb}
\usepackage{amsthm}
\usepackage{nicefrac}
\usepackage{microtype}
\usepackage{xcolor}
\usepackage{graphicx}
\graphicspath{{figures/}}

\usepackage{hyperref}       
\hypersetup{
	colorlinks   = true, 
	urlcolor     = blue, 
	linkcolor    = blue, 
	citecolor   = blue 
}

\newcommand{\Pp}{\mathbb{P}}
\newcommand{\E}{\mathbb{E}}

\newcommand{\ind}{\mathbf{1}}

\theoremstyle{plain}
\newtheorem{theorem}{Theorem}
\newtheorem{proposition}{Proposition}
\newtheorem{lemma}{Lemma}

\theoremstyle{definition}

\newtheorem{assumption}{Assumption}

\theoremstyle{remark}
\newtheorem{remark}{Remark}

\title{Conformal-Style Quantile Analyses for Stochastic Bandits}

\author{%
  Chengyu Du\\
  University of Massachusetts Amherst \\
  \texttt{chengyudu@umass.edu} \\
  \And
  Mengfan Xu\\
  University of Massachusetts Amherst \\
  \texttt{mengfanxu@umass.edu} \\
}

\begin{document}

\maketitle

\begin{abstract}
Stochastic bandit algorithms are usually analyzed under a mean-reward
criterion, yet many problems favor arms with strong upper-tail performance,
which we study herein. For a fixed miscoverage level \(\alpha\), the natural
upper-tail target of arm \(j\) is the upper endpoint
\(F_j^{-1}(1-\alpha/2)\) of a central prediction interval. This target can
rank arms differently from their means, creating a central  mismatch with the
classical bandit objective. To this end, we propose ACP-UCB1, a
conformal-style policy that combines an adaptive conformal estimate of the
upper endpoint with a UCB-type optimism bonus. The technical challenge is that the conformity scores used by ACP-UCB1 are recomputed from evolving empirical quantile estimates and evaluated at an adaptive level. We control this endpoint through reward-quantile concentration, a perturbation argument for recomputed score quantiles, and deterministic localization of the adaptive level. ACP-UCB1 achieves logarithmic upper-quantile regret with per-arm
contribution \(O(\nicefrac{\log n}{\Delta_j^{\mathrm{ACP}}})\). We also
provide metric-specific regret decompositions comparing ACP-UCB1 with UCB1 and
use numerical experiments to validate performance and improvement.\end{abstract}

\section{Introduction}
\label{sec:intro}

Multi-armed bandits (MAB) form a classical framework for sequential decision making under partial feedback.  In the stochastic formulation, a learner interacts with a finite collection of arms, pulls one arm at each round, and observes only the reward generated by that arm. The objective is usually stated in terms of regret relative to the arm with the largest expected reward. This mean-reward criterion is the foundation of the finite-time analysis of UCB1 and
many of its refinements, including KL-UCB, variance-adaptive UCB, Thompson sampling, and the modern theory of stochastic bandits \citep{lai1985asymptotically,auer2002finite,thompson1933likelihood,garivier2011klucb,audibert2009exploration,lattimore2020bandit}. The implicit modeling assumption in this line of work is that the arm with the best long-run average reward is also the desired arm for the decision maker.  Yet several bandit objectives depend on distributional features beyond the mean, including risk-sensitive and quantile-based criteria \citep{sani2012risk,galichet2013exploration,szorenyi2015qualitative,david2016pure}. This distinction is especially relevant when decisions are made according to predictive upper-tail performance. In such settings, two arms may have similar means but very different upper predictive behavior, and the mean-optimal arm may not be the arm preferred by an upper-tail decision rule. More generally, an arm may have a moderate mean but still produce substantially larger high-end outcomes than another arm with a slightly larger mean.

In fact, non-mean objectives have been studied in several bandit settings. Risk-aware bandits replace expected reward by distributional criteria such as mean-variance tradeoffs, CVaR, or other risk-sensitive functionals, while
quantile bandits study learning and identification when a prescribed reward quantile is the object of interest
\citep{sani2012risk,szorenyi2015qualitative,nikolakakis2021quantile}. These works establish an important point for the present paper: changing the reward functional changes both the identity of the optimal arm and the
statistical analysis required to learn it. Our objective is not specified directly as a fixed risk functional of the reward distribution. It is induced by a score-calibrated predictive construction: starting from a central reward interval, we score deviations from that interval and add a score-quantile correction to its upper endpoint.  

The methodological distinction we exploit comes from conformal prediction. Rather than estimating uncertainty around an unknown parameter, conformal methods calibrate prediction sets for future observations through scores
\citep{vovk2005algorithmic,lei2018distribution,angelopoulos2023gentle}. A conformal construction begins with a baseline predictive rule, assigns a score to each candidate observation, and calibrates the rule by taking a score quantile. Conformalized quantile regression is a representative example: a preliminary lower--upper quantile interval is corrected by a calibrated score to form a prediction interval \citep{romano2019conformalized}. Adaptive and online conformal methods further study how calibration levels or score thresholds can be updated sequentially or under departures from exchangeability \citep{gibbs2021adaptive,zaffran2022adaptive,angelopoulos2024online,barber2023conformal}.

However, the bandit setting changes the role of calibration. In standard conformal prediction, the calibration sample is fixed independently of the prediction target, or the full label is observed after prediction. In a bandit problem, observations are arm-local, adaptively selected, and only available for pulled arms. Thus the empirical score distribution used to calibrate an endpoint is itself produced by the sampling policy. This creates a different statistical object: the endpoint must be accurate enough for arm selection, and its error must be controlled uniformly over the random arm histories generated by the policy.

Recent work has begun to study conformal prediction under partial feedback and in bandit-type decision problems. In stochastic online conformal prediction with semi-bandit feedback, the learner constructs prediction sets while observing the true label only when it is contained in the set; the resulting regret is measured relative to an optimal conformal predictor \citep{ge2025stochastic}. Conformal Bandits combine conformal prediction with bandit policies to study statistical validity and reward efficiency, with particular emphasis on small-gap regimes \citep{cuonzo2025conformal}. These works address conformal calibration under sequential or partial-feedback structures, but their statistical object differs from the one considered here. We formulate a stochastic reward-arm problem in which the value of an arm is induced by a score-calibrated predictive endpoint. For each arm, a central reward interval defines a conformity score; a score quantile then corrects the interval, and the corrected upper endpoint is used as the arm value. This formulation separates our objective from existing conformal-bandit problems. The conformal construction is not used only to certify coverage around a policy, nor is the bandit objective specified as a fixed risk functional in advance. Instead, the score-calibrated endpoint itself defines the population arm value, and the learning problem is to identify this value under arm-local partial feedback.

To this end, we propose ACP-UCB1 to learn this endpoint under bandit feedback. For each arm, the algorithm estimates the upper reward anchor of a central interval, adds an empirical score-quantile correction computed from recomputed arm-local scores, and then adds a UCB-type exploration bonus. Equivalently, after \(s\) pulls of arm \(j\), the selection index has three parts: an upper reward-anchor estimate, a conformal correction, and an optimism bonus of order \(\sqrt{\log t/s}\). 

Under i.i.d. rewards and local density regularity around the relevant reward and score quantiles, our main theoretical result shows that ACP-UCB1 achieves logarithmic regret  with order \(O(\nicefrac{\log n}{\Delta_j^{\mathrm{ACP}}})\) for the per arm. We also give a formal comparison witha mean-based UCB benchmark.  In a Gaussian model, the population ACP endpoint at nominal level \(\alpha\) is \(\mu_j+z_{1-\alpha/2}\sigma_j\). Thus the mean objective and the ACP upper-tail objective are two endpoints of the utility path \(v_j(\kappa)=\mu_j+\kappa\sigma_j\). When the two objectives share the same optimizer, the comparison is governed by the logarithmic pull-count constants; ACP-UCB1 has the smaller leading coefficient under both endpoint metrics when its ACP gaps are large enough that \(\gamma_j^{\rm ACP}<\gamma_j^{\rm UCB}\) for the relevant suboptimal arms. When the objectives select different arms, the mismatch is more severe: concentrating on one objective yields a linear leading term under the other. The comparison clarifies ACP-UCB1 as a policy for the conformal upper-tail objective rather than as a proxy for mean-regret minimization.

The experiments are designed to test the metric dependence predicted by the analysis. On Gaussian instances with different mean-optimal and ACP-optimal arms, the two algorithms exhibit opposite regret profiles: ACP-UCB1 performs better under the ACP upper-tail metric, while mean-based UCB performs better under mean regret. A sweep over \(\alpha\) confirms that smaller miscoverage levels amplify the role of reward variability in the upper-tail objective. Experiments beyond Gaussian rewards show the same qualitative pattern for heavier-tailed and skewed reward families. The empirical results therefore illustrate the central message of the paper: changing the arm value changes both the target arm and the appropriate bandit policy.

The contributions of this paper are summarized as follows. First, we formulate a stochastic upper-tail bandit problem in which conformal score calibration defines the arm value: a central reward interval is corrected by a score quantile, and the resulting upper endpoint is optimized by the learner. Second, we introduce ACP-UCB1, a conformal-style UCB policy whose index combines an empirical upper reward anchor, a score-quantile correction, and a UCB-type exploration bonus. Third, we prove a  logarithmic pseudo-regret bound for the ACP upper-tail objective. Fourth, we provide a Gaussian comparison with mean-based UCB, separating the common-optimizer regime, where the comparison reduces to logarithmic constants, from the objective-mismatch regime, where evaluating a policy under the other metric produces a linear leading term. Finally, we present numerical experiments that illustrate the metric-specific behavior predicted by the theory across Gaussian, heavy-tailed, and skewed reward families.

\section{Problem Formulation}
\label{sec:setup}

In this section, we formally introduce the upper-tail bandit problem induced by the score-calibrated endpoint and notations used throughout the paper. Specifically, the formulation specifies (i) the population arm value and (ii) the regret criterion to be minimized.

\subsection{Stochastic Bandit Model}
\label{subsec:bandit}

Throughout, we consider a stochastic \(K\)-armed bandit with arm set \(\mathcal A=\{1,\ldots,K\}\). At each round \(t\geq 1\), the learner selects an arm \(I_t\in\mathcal A\) and subsequently observes a reward \(Y_t\in\mathbb R\). For arm \(j\), we let \(\tau_{j,s}\) denote the global time of its \(s\)-th pull and write \(Y_{j,s}=Y_{\tau_{j,s}}\). The number of pulls of arm \(j\) up to time \(t\) is denoted by \(T_j(t)=\sum_{r=1}^t\ind\{I_r=j\}\). We assume that, for each arm \(j\), the local reward sequence \(Y_{j,1},Y_{j,2},\ldots\) is i.i.d. from a continuous distribution \(F_j\) with density \(f_j\).

\subsection{Score-Based Conformal Correction}
\label{subsec:population_conformal}

We first describe the population quantity that defines the bandit objective. Fix a target miscoverage level \(\alpha\in(0,1)\), and set \(\tau_{\rm low}=\alpha/2\) and \(\tau_{\rm high}=1-\alpha/2\). For arm \(j\), let \(q_j^{\rm low}=F_j^{-1}(\tau_{\rm low})\) and \(q_j^{\rm high}=F_j^{-1}(\tau_{\rm high})\). The interval \([q_j^{\rm low},q_j^{\rm high}]\) is the nominal central interval of the reward distribution. For endpoints \(\ell\le u\), we use the score
\[
s(y;\ell,u)=\max\{\ell-y,\,y-u\},
\]
so that \(s(y;\ell,u)\le z\) is equivalent to \(y\in[\ell-z,u+z]\). Thus the score records how much the central interval must be expanded in order to cover \(y\). Using the base quantiles of arm \(j\), define \(S_j:=s(Y_j;q_j^{\rm low},q_j^{\rm high})
     =\max\{q_j^{\rm low}-Y_j,\,Y_j-q_j^{\rm high}\}, Y_j\sim F_j.\) Let \(G_j\) denote the distribution function of \(S_j\). Since the lower endpoint of the score support is \(z_{j,\min}:=\frac{q_j^{\rm low}-q_j^{\rm high}}{2},\) we have, for \(z\ge z_{j,\min}\),\(G_j(z):=\Pp(S_j\le z)
=F_j(q_j^{\rm high}+z)-F_j(q_j^{\rm low}-z),\) and \(G_j(z)=0\) for \(z<z_{j,\min}\). On the interior of the support of \(S_j\), \(G_j\) has density
\begin{equation}
g_j(z)=f_j(q_j^{\rm high}+z)+f_j(q_j^{\rm low}-z).
\label{eq:gj_def}
\end{equation}

For a calibration level \(\beta\in[\underline\alpha,\overline\alpha]\), let \(c_j(\beta)=G_j^{-1}(1-\beta)\). This gives the population conformal interval \(\mathcal C_j(\beta)=[q_j^{\rm low}-c_j(\beta),q_j^{\rm high}+c_j(\beta)]\), whose upper endpoint is \(V_j(\beta)=q_j^{\rm high}+c_j(\beta)\).

The arm value used in this paper is this upper endpoint, not the mean reward. At the nominal level \(\beta=\alpha\), the base interval already has probability mass \(1-\alpha\), hence \(G_j(0)=1-\alpha\). Under Assumption~\ref{assum:density}, the corresponding score quantile is unique and \(c_j(\alpha)=0\). We therefore write
\[
u_j^\star = V_j(\alpha)=q_j^{\rm high}=F_j^{-1}\!\left(1-\frac{\alpha}{2}\right)
\]
for the nominal upper-tail value of arm \(j\).

\subsection{Upper-Tail Pseudo-Regret}
\label{subsec:regret_objective_agreement}

The performance criterion associated with ACP-UCB1 is pseudo-regret under the population upper-tail values:
\[
R_n^{\rm ACP}
=
n\max_{j\in\mathcal A}u_j^\star
-
\E[\sum_{t=1}^n u_{I_t}^\star].
\]
If the upper-tail optimal arm \(A\) is unique, this becomes
\[
R_n^{\rm ACP}
=
\sum_{j\ne A}\Delta_j^{\rm ACP}\E[T_j(n)],
\qquad
\Delta_j^{\rm ACP}=u_A^\star-u_j^\star .
\]
For comparison, the classical mean objective uses \(M\in\arg\max_j\mu_j\) and gaps \(\Delta_j^\mu=\mu_M-\mu_j\). The cases \(A=M\) and \(A\ne M\) separate instances where the two objectives agree from those where upper-tail optimization and mean optimization lead to different target arms.

\section{Methodology}
\label{sec:policy}


The population endpoint \(V_j(\beta)\) defined above is not available to the learner, since both the reward quantiles and the score distribution are unknown. To address it, ACP-UCB1 replaces this term by an arm-local endpoint constructed from the rewards observed from each arm.  This section describes this empirical construction and the decision rule. We first define the empirical conformal endpoint based on recomputed scores and score-quantile correction in Section \ref{subsec:empirical_conformal}. Then we specify the localized adaptive update for the calibration level in Section \ref{subsec:adaptive_level_update}. Finally, we add a UCB-type exploration bonus to the corrected upper endpoint to build the ACP-UCB1 arm-selection rule in Section \ref{subsec:acp_ucb}.

\subsection{Empirical Conformal Endpoint}
\label{subsec:empirical_conformal}

We now describe the arm-local statistic available under bandit feedback.  For a sample \(x_{1:m}:=(x_1,\ldots,x_m)\), define the empirical quantile \(\widehat Q_m(\tau;x_{1:m})
:=
\inf\left\{
x\in\mathbb R:
\frac{1}{m}\sum_{r=1}^m \ind\{x_r\le x\}\ge \tau
\right\}.\)
For arm \(j\) and local time \(s\ge2\), let
\(Y_{j,1:s-1}:=(Y_{j,1},\ldots,Y_{j,s-1})\). The empirical base endpoints are
\begin{equation}
\hat q^{\rm low}_{j,s-1}
:=
\widehat Q_{s-1}(\tau_{\rm low};Y_{j,1:s-1}),
\qquad
\hat q^{\rm high}_{j,s-1}
:=
\widehat Q_{s-1}(\tau_{\rm high};Y_{j,1:s-1}).
\label{eq:emp_quantiles}
\end{equation}
These quantities estimate the population reward quantiles
\(q_j^{\rm low}\) and \(q_j^{\rm high}\). Given these empirical endpoints, all past observations of arm \(j\) are rescored relative to the same current interval:
\begin{equation}
\widetilde S^{(s)}_{j,r}
:=
s\!\left(Y_{j,r};\hat q^{\rm low}_{j,s-1},\hat q^{\rm high}_{j,s-1}\right)
=
\max\{\hat q^{\rm low}_{j,s-1}-Y_{j,r},
      Y_{j,r}-\hat q^{\rm high}_{j,s-1}\},
\qquad 1\le r<s .
\label{eq:recomputed_scores}
\end{equation}
The dependence among the scores in \eqref{eq:recomputed_scores} arises solely
from the common empirical endpoints.
For the arm-specific adaptive level \(\alpha_{j,s}\), define the empirical
score correction by
\begin{equation}
\widehat c^{\rm conf}_{j,s}
:=
\widehat Q_{s-1}
\left(
1-\alpha_{j,s};
\widetilde S^{(s)}_{j,1},\ldots,\widetilde S^{(s)}_{j,s-1}
\right).
\label{eq:conf_quantile}
\end{equation}
Equivalently, \(\widehat c^{\rm conf}_{j,s}
=
\inf\left\{z\in\mathbb R:
\frac{1}{s-1}\sum_{r=1}^{s-1}
\ind\{\widetilde S^{(s)}_{j,r}\le z\}
\ge 1-\alpha_{j,s}
\right\}.\)
The corresponding empirical conformal interval is
\begin{equation}
L_{j,s}
=
\hat q^{\rm low}_{j,s-1}-\widehat c^{\rm conf}_{j,s},
\qquad
U_{j,s}
=
\hat q^{\rm high}_{j,s-1}+\widehat c^{\rm conf}_{j,s}.
\label{eq:interval_def}
\end{equation}
The upper endpoint \(U_{j,s}\) is the exploitation statistic in the ACP-UCB1
index; the optimism bonus is added only in the arm-selection rule. The analysis compares the plug-in scores in \eqref{eq:recomputed_scores} with the oracle scores \(S^{\rm pop}_{j,r}
:=
s(Y_{j,r};q_j^{\rm low},q_j^{\rm high})
=
\max\{q_j^{\rm low}-Y_{j,r},Y_{j,r}-q_j^{\rm high}\}.\)
Since \(s(y;\ell,u)\) is Lipschitz in the endpoints, the plug-in scores differ from the oracle scores by at most the endpoint estimation error: 
\begin{equation}
\bigl|
\widetilde S^{(s)}_{j,r}
-
S^{\rm pop}_{j,r}
\bigr|
\le
\max\left\{
\bigl|\hat q^{\rm low}_{j,s-1}-q_j^{\rm low}\bigr|,
\bigl|\hat q^{\rm high}_{j,s-1}-q_j^{\rm high}\bigr|
\right\}
\le
\bigl|\hat q^{\rm low}_{j,s-1}-q_j^{\rm low}\bigr|
+
\bigl|\hat q^{\rm high}_{j,s-1}-q_j^{\rm high}\bigr|,
 r<s.
\label{eq:score_plugin_bound}
\end{equation}

\subsection{Adaptive Level Update}
\label{subsec:adaptive_level_update}

The calibration level is updated on the local time scale of each arm. After the interval \([L_{j,s},U_{j,s}]\) is formed and \(Y_{j,s}\) is observed, define the online score \(S_{j,s}^{\rm on}=s(Y_{j,s};\hat q^{\rm low}_{j,s-1},\hat q^{\rm high}_{j,s-1})\). The miss indicator is
\[
\mathrm{err}_{j,s}
=
\ind\{Y_{j,s}\notin[L_{j,s},U_{j,s}]\}
=
\ind\{S_{j,s}^{\rm on}>\widehat c^{\rm conf}_{j,s}\}.
\]
Starting from \(\alpha_{j,2}\in[\underline\alpha,\overline\alpha]\), ACP-UCB1 performs the projected update
\[
\alpha_{j,s+1}
=
\Pi_{\mathcal I_s^\alpha}
\left[
\alpha_{j,s}
+
\eta_{j,s}(\alpha-\mathrm{err}_{j,s})
\right],
\qquad
\mathcal I_s^\alpha
=
[\underline\alpha,\overline\alpha]\cap[\alpha-\lambda\psi_s,\alpha+\lambda\psi_s],
\]
where \(\psi_s=\sqrt{\log(s)/s}\). The update is adaptive through the realized coverage error, while the localized projection keeps the level within a shrinking neighborhood of \(\alpha\), which enforces
\[
|\alpha_{j,s+1}-\alpha|\le \lambda\sqrt{\frac{\log s}{s}}.
\]

\subsection{Arm Selection Rule}
\label{subsec:acp_ucb}

With these quantities at hand, ACP-UCB1 follows an optimism rule for the upper conformal endpoint. After an initialization phase in which each arm is pulled \(
s_0=s_{\rm loc}(n)
:=\left\lceil\frac{b^2\log n}{16\rho^2}\right\rceil\vee2
\) times, the learner selects
\[
I_t\in\arg\max_{j\in\mathcal A}
\{
U_{j,T_j(t-1)+1}
+
b\sqrt{\frac{\log t}{T_j(t-1)}}
\}.
\]
Here \(b>0\) is the exploration coefficient, and \(\rho=\min\{r_s,r_y/2\}\) is determined by the local regularity constants in Assumption~\ref{assum:density}.

\section{Theoretical Analysis}
\label{sec:dep}

We next give a finite-time guarantee for ACP-UCB1. The proof hints on a new analytical step on concentration: the index contains an empirical reward quantile and an empirical score quantile, both computed from the same arm-local sample. The assumptions below impose only local regularity near the reward and score quantiles that enter the endpoint.

\subsection{Local Regularity Assumptions}
\label{subsec:assumptions}

\begin{assumption}
\label{assum:density}
There exist constants \(r_y,r_s>0\), \(0<f_{\min}\le f_{\max}<\infty\), and
\(g_{\min}>0\), independent of \(j\), such that the following conditions hold for
each arm \(j\):

(a) The reward density satisfies \(f_{\min}\le f_j(x)\le f_{\max}\) on $\mathcal I_j^{\rm rew}
:=
\cup_{\beta\in[\underline\alpha,\overline\alpha]}
\left(
[q_j^{\rm low}-c_j(\beta)-r_y,\,
 q_j^{\rm low}-c_j(\beta)+r_y]
\cup
[q_j^{\rm high}+c_j(\beta)-r_y,\,
 q_j^{\rm high}+c_j(\beta)+r_y]
\right).$

(b) The score density \(g_j\) in \eqref{eq:gj_def} satisfies
\(g_j(z)\ge g_{\min}\) on
\[
\mathcal I_j^{\rm score}
:=
\cup_{\beta\in[\underline\alpha,\overline\alpha]}
[c_j(\beta)-r_s,\,c_j(\beta)+r_s].
\]
We write \(
h_{\min}:=\min\{f_{\min},g_{\min}\},
\qquad
\rho:=\min\{r_s,r_y/2\}.\)
\end{assumption}

Condition (a) provides local control of the reward quantiles, while condition (b) converts score distribution errors into score-quantile errors. No global parametric assumption is imposed on the reward distributions.

\begin{remark}
For Gaussian rewards with uniformly bounded standard deviations,
Assumption~\ref{assum:density} holds with constants independent of the arm.
The details are in Appendix~\ref{app:proof_normal_local_constants}.
\end{remark}

\subsection{Main Regret Guarantee}
The proof relies on two concentration ingredients, stated in Appendix~\ref{app:aux_concentration_facts}: local concentration of empirical reward quantiles and concentration of the recomputed score-quantile correction. Combining these ingredients with the standard UCB argument gives the following guarantee.

\begin{theorem}
\label{thm:main_thm}
Suppose every arm satisfies the i.i.d. reward model and
Assumption~\ref{assum:density}. Assume that the upper-tail-optimal arm
\(A\in\arg\max_{j\in\mathcal A}u_j^\star\) is unique. ACP-UCB1 uses the
recomputed scores \eqref{eq:recomputed_scores}, the localized projected update, and the index rule. Suppose that
\(0<\lambda\le b g_{\min}/4\)
 and for each \(j\ne A\), assume
\(r_y\ge \Delta_j^{\rm ACP}/4,
r_s\ge \Delta_j^{\rm ACP}/4.\) Set
\(b_0:=\frac{64}{h_{\min}}.\)
If \(b\ge b_0\), \(s_0=s_{\rm loc}(n)\), and \(n\ge Ks_0\), then for every
\(j\ne A\),
\begin{equation}
\E[T_j(n)]
\le
\frac{4b^2\log n}{(\Delta_j^{\rm ACP})^2}+B_j,
\label{eq:thm1_main}
\end{equation}
where \(B_j<\infty\) is independent of \(n\).

Consequently,
\[
R_n^{\rm ACP}(ACP-UCB1)
\le
\sum_{j\ne A}
\frac{4b^2\log n}{\Delta_j^{\rm ACP}}
+O(1).
\]
\end{theorem}

The proof is deferred to Appendix~\ref{app:proof_main}.

\subsection{Comparison With Mean-Based UCB}
\label{subsec:metric_comparison}

The preceding theorem is stated for the ACP upper-tail objective. We now compare this objective with the classical mean-reward objective in a Gaussian model, where the distinction can be written explicitly. Specifically, we denote the mean-based Gaussian UCB benchmark by UCB1-NORMAL in the theoretical comparison, and refer to the same benchmark as UCB1 in the experiments. This comparison is not meant to claim that one policy dominates the other uniformly. Rather, it shows that the two policies are optimized for different population criteria, and that evaluating a policy under the other criterion may create a linear loss.

\subsubsection{A One-Parameter Gaussian Utility Path}

For Gaussian rewards, the ACP upper endpoint admits the representation
\(u_j^\star=\mu_j+z_{1-\alpha/2}\sigma_j\). 

Writing \(a=z_{1-\alpha/2}\), consider the utility path
\[
v_j(\kappa):=\mu_j+\kappa\sigma_j,
\qquad 0\le\kappa\le a.
\]
The endpoint \(\kappa=0\) recovers the mean objective, while \(\kappa=a\) recovers the ACP upper-tail objective. Thus \(\kappa\) continuously interpolates between mean seeking and upper-tail seeking behavior. 
Let
\(v_\star(\kappa):=\max_j v_j(\kappa),
    D_g(\kappa):=v_\star(\kappa)-v_g(\kappa)\)
and define the \(\kappa\)-utility pseudo-regret of a policy \(\pi\) by
\begin{equation}
    R_n^\kappa(\pi)
    :=
    \sum_{t=1}^n
    \E\bigl[v_\star(\kappa)-v_{I_t}(\kappa)\bigr].
    \label{eq:kappa_regret}
\end{equation}
Write \(
    A\in\arg\max_j v_j(a),
    M\in\arg\max_j v_j(0)\) for the ACP-optimal and mean-optimal arms, and assume uniqueness whenever the
notation is used.

\begin{proposition}
\label{prop:objective_transfer}
Fix a Gaussian bandit instance and an arm \(g\). If a policy \(\pi_g\) satisfies \(
\sum_{h\ne g}\E_{\pi_g}[T_h(n)]=O(\log n),\)
then
\[
\sup_{\kappa\in[0,a]}
\left|R_n^\kappa(\pi_g)-D_g(\kappa)n\right|
=O(\log n).
\]
Consequently, if ACP-UCB1 and UCB1-NORMAL satisfy the preceding condition with \(g=A\) and \(g=M\), respectively, then uniformly over \(\kappa\in[0,a]\),
\begin{align}
    R_n^\kappa(\mathrm{ACP\text{-}UCB1})
    &=D_A(\kappa)n+O(\log n),
    \label{eq:acp_kappa_leading}\\
    R_n^\kappa(\mathrm{UCB1\text{-}NORMAL})
    &=D_M(\kappa)n+O(\log n).
    \label{eq:mean_kappa_leading}
\end{align}
\end{proposition}
The proof is given in Appendix~\ref{app:normal_comparison_proofs}.

\subsubsection{Metric Trade-Off}

We first consider the objective-mismatch regime \(A\ne M\). In this case, the mean-optimal arm has larger mean, \(\mu_M>\mu_A\), but the ACP-optimal arm compensates through a larger standard deviation. Let \(\delta_\mu=\mu_M-\mu_A\) and \(\delta_\sigma=\sigma_A-\sigma_M\). The condition that arm \(A\) is preferred at the ACP endpoint is \(0<\delta_\mu<a\delta_\sigma\), so \(\delta_\sigma>0\) and the two objectives switch preference at \(\kappa_0=\delta_\mu/\delta_\sigma\in(0,a)\).

Assume
\[
v_\star(\kappa)=\max\{v_A(\kappa),v_M(\kappa)\},
\qquad 0\le\kappa\le a.
\]
Then
\begin{equation}
    D_A(\kappa)=(\delta_\mu-\kappa\delta_\sigma)_+,
    \qquad
    D_M(\kappa)=(\kappa\delta_\sigma-\delta_\mu)_+ .
    \label{eq:two_arm_D_functions}
\end{equation}
Combining \eqref{eq:two_arm_D_functions} with
Proposition~\ref{prop:objective_transfer} gives the endpoint comparison:

At the ACP endpoint,
\begin{align}
    R_n^a(\mathrm{ACP\text{-}UCB1})
    &=O(\log n), R_n^a(\mathrm{UCB1\text{-}NORMAL})
    =(a\delta_\sigma-\delta_\mu)n+O(\log n).
\end{align}
At the mean endpoint,
\begin{align}
    R_n^0(\mathrm{ACP\text{-}UCB1})
    &=\delta_\mu n+O(\log n), R_n^0(\mathrm{UCB1\text{-}NORMAL})
    =O(\log n).
\end{align}

Thus the separation is not caused by a weaker concentration inequality. It is the deterministic cost of allocating most pulls to the arm favored by the other population criterion.

The same two deterministic functions describe aggregate comparisons over the
whole path.  From \eqref{eq:two_arm_D_functions},
\begin{equation}
\sup_{\kappa\in[0,a]}D_A(\kappa)=\delta_\mu,
    \qquad
\sup_{\kappa\in[0,a]}D_M(\kappa)=a\delta_\sigma-\delta_\mu.
    \label{eq:two_arm_suprema}
\end{equation}
Hence ACP-UCB1 has the smaller worst-case leading coefficient over \([0,a]\) if
and only if
\begin{equation}
    2\delta_\mu<a\delta_\sigma .
    \label{eq:minimax_kappa_condition}
\end{equation}
For a nominal \(95\%\) central prediction interval, \(\alpha=0.05\) and
\(a=z_{0.975}\approx1.96\). The condition becomes \(\sigma_A-\sigma_M >
\frac{2}{1.96}(\mu_M-\mu_A)
\approx 1.02(\mu_M-\mu_A).\)

\subsubsection{Logarithmic Constants Comparison}

We finally consider the common-optimizer regime, where the mean and ACP objectives select the same arm. In this case, both policies concentrate on the same target arm and both regrets are logarithmic under both metrics. The comparison is therefore reduced to the leading constants in the pull-count bounds.

Assume \(A\) is the unique maximizer of both \(v_j(0)\) and \(v_j(a)\).  For every
\(j\ne A\), define
\begin{equation}
    \Delta_j^\mu:=\mu_A-\mu_j,
    \qquad
    \Delta_j^{\rm ACP}:=(\mu_A+a\sigma_A)-(\mu_j+a\sigma_j).
    \label{eq:common_optimizer_gaps}
\end{equation}
That is, \( \Delta_j^{\rm ACP}
    =\Delta_j^\mu+a(\sigma_A-\sigma_j).\) Thus the two objectives share the optimizer, but they do not in general induce
the same gaps.

For ACP-UCB1, Theorem~\ref{thm:main_thm} gives, under the same
gap-radius condition,
\begin{equation}
\E[T_j^{\rm ACP}(n)]
\le
\gamma_j^{\rm ACP}\log n+B_j^{\rm ACP},
\qquad
\gamma_j^{\rm ACP}
:=
\frac{4b^2}{(\Delta_j^{\rm ACP})^2}.
\label{eq:acp_pull_envelope_common}
\end{equation}

For UCB1-NORMAL, Theorem~4 of \citet{auer2002finite} gives the aggregate mean
regret bound
\begin{equation}
R_n^0(\mathrm{UCB1\text{-}NORMAL})
\le
\sum_{j\ne A}
\left(
\frac{256\sigma_j^2}{\Delta_j^\mu}
+
8\Delta_j^\mu
\right)\log n
+
\left(1+\frac{\pi^2}{2}\right)
\sum_{j\ne A}\Delta_j^\mu .
\label{eq:auer_ucb1_normal_theorem4}
\end{equation}
Correspondingly, 
\begin{equation}
\E[T_j^{\rm UCB}(n)]
\le
\gamma_j^{\rm UCB}\log n+B_j^{\rm N},
\qquad
\gamma_j^{\rm UCB}
:=
\max\left\{
\frac{256\sigma_j^2}{(\Delta_j^\mu)^2},
8
\right\}.
\label{eq:ucb_normal_pull_envelope_common}
\end{equation}

Under the ACP endpoint, the two armwise leading terms are \(\Delta_j^{\rm ACP}\gamma_j^{\rm ACP}
    \text{and}
    \Delta_j^{\rm ACP}\gamma_j^{\rm UCB};\)
under the mean endpoint, they are \( \Delta_j^\mu\gamma_j^{\rm ACP}
    \text{and}
    \Delta_j^\mu\gamma_j^{\rm UCB}.\)

Since the endpoint gaps are positive, the same armwise inequality determines the comparison under both metrics: ACP-UCB1 has the smaller leading term for arm \(j\) precisely when \(\gamma_j^{\rm ACP}<\gamma_j^{\rm UCB}\). Equivalently, the ACP gap enlargement must be large enough to offset the larger constant introduced by estimating an upper conformal endpoint rather than a mean. That is,
\begin{equation}
\frac{4b^2}{(\Delta_j^{\rm ACP})^2}
<
\max\left\{
\frac{256\sigma_j^2}{(\Delta_j^\mu)^2},8
\right\}.
\label{eq:common_optimizer_exact_condition}
\end{equation}
Equivalently, with \(
r_j:=\frac{\Delta_j^{\rm ACP}}{\Delta_j^\mu},\)
this condition becomes
\(r_j > \frac{2b}
{\sqrt{\max\{256\sigma_j^2,\,8(\Delta_j^\mu)^2\}}}.\)
Thus, in the common-optimizer regime, both policies are logarithmic under both
metrics; the comparison is determined by whether the ACP gap enlargement
outweighs its endpoint-estimation constant.

\section{Numerical Experiments}
\label{sec:experiments}

We conduct numerical experiments to illustrate the metric-specific behavior predicted by the preceding analysis. Specifically, we first demonstrate the performance of ACP-UCB1 in comparison with UCB1 on a Gaussian instance where the mean-optimal and ACP-optimal arms are different. Moreover, we examine how the regret changes with respect to the nominal miscoverage level \(\alpha\). Furthermore, we conduct experiments beyond Gaussian rewards, including heavy-tailed and skewed reward families, to illustrate the robustness of the observed phenomenon. Details about the data generation and implementation are provided in Appendix~\ref{app:experiment_details}.

\subsection{Metric Comparison Results}

The results for the Gaussian instance are presented in Figure~\ref{fig:normal_cumulative}. The two panels evaluate the same trajectories under cumulative ACP-regret and cumulative mean-regret, respectively. ACP-UCB1 has smaller ACP-regret, while mean-based UCB1 has smaller mean-regret. This is consistent with the comparison in Subsection~\ref{subsec:metric_comparison}: when the two objectives select different arms, a policy optimized for one objective may incur a linear transfer cost under the other objective.

\begin{figure}[t]
    \centering
    \includegraphics[width=1\textwidth]{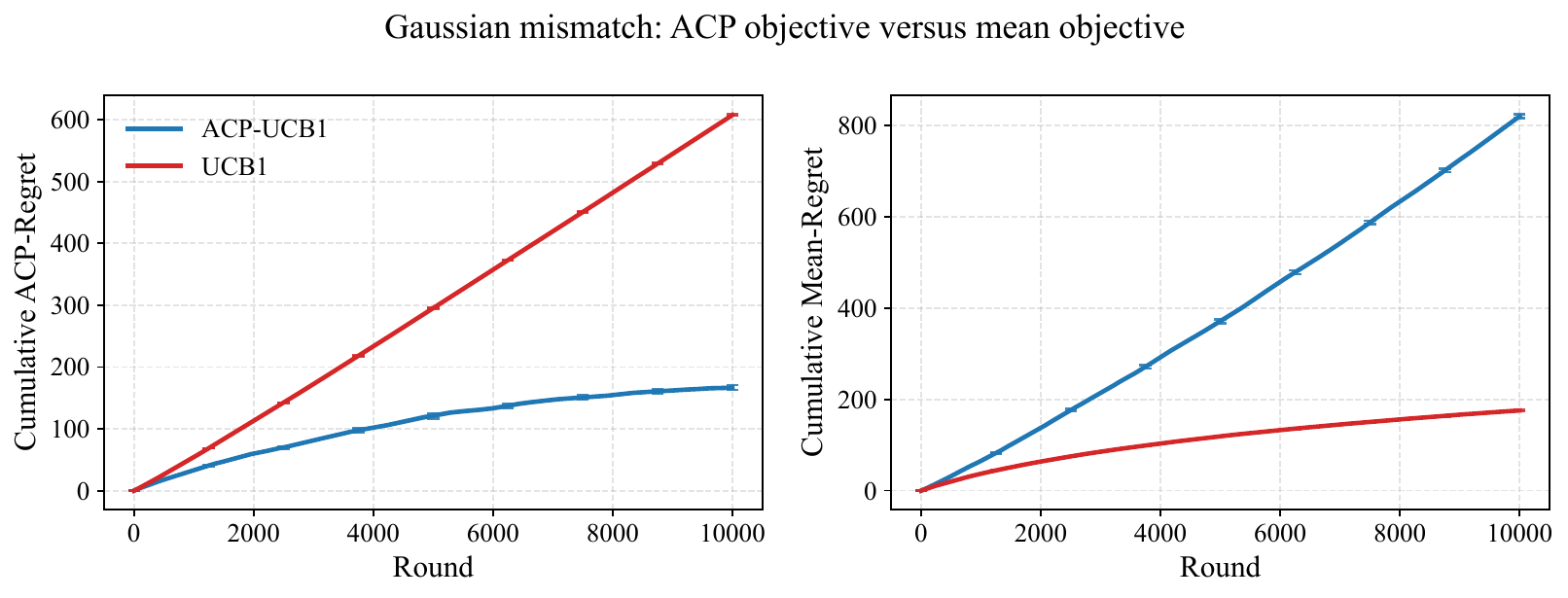}
    \caption{
   Metric comparison on a Gaussian objective-mismatch instance.
    }
    \label{fig:normal_cumulative}
\end{figure}

\subsection{Nominal-Level Dependency Results}

We next examine the dependence on the nominal level \(\alpha\). The results are shown in Figure~\ref{fig:alpha_sweep}. As \(\alpha\) decreases, the ACP objective moves farther into the upper tail and becomes more sensitive to reward variability. In this regime, ACP-UCB1 continues to have smaller final ACP-regret, while mean-based UCB1 continues to have smaller final mean-regret. The ACP-regret of mean-based UCB1 increases sharply for smaller \(\alpha\), reflecting the increasing mismatch between mean optimization and upper-tail optimization.

\begin{figure}[t]
    \centering
    \includegraphics[width=1\textwidth]{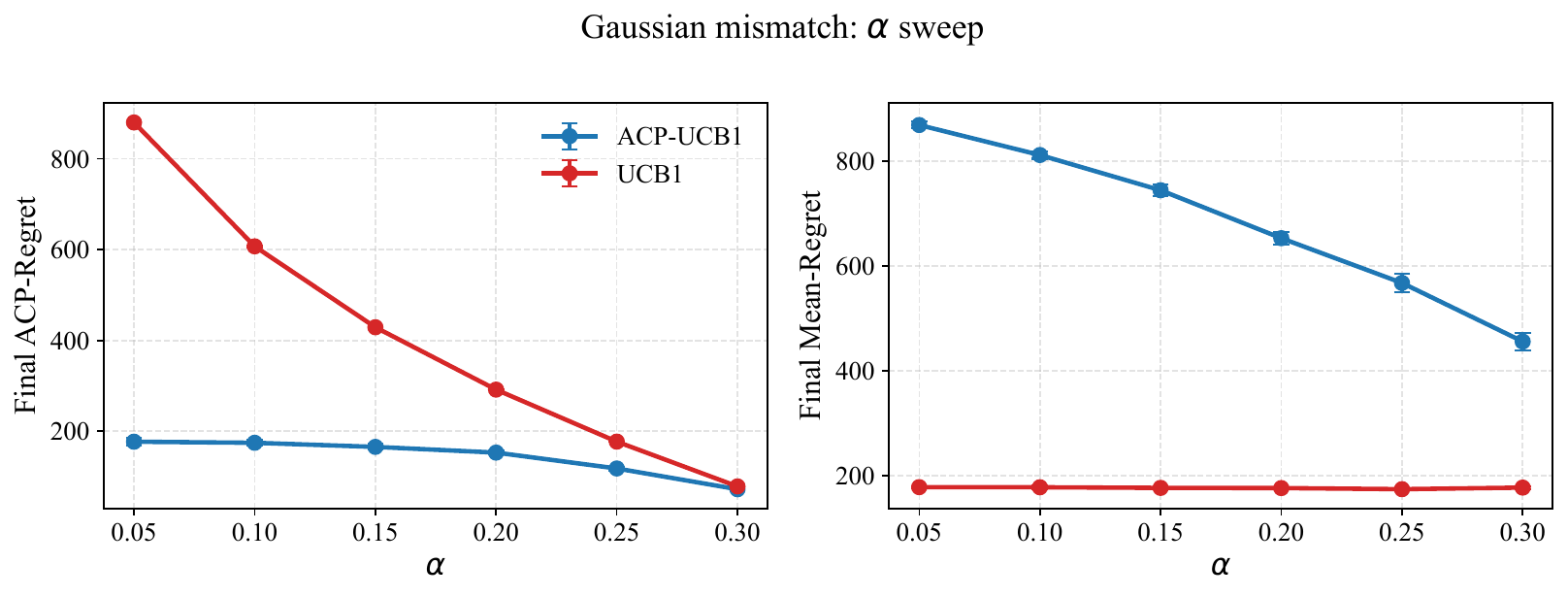}
    \caption{
    Final regret as the nominal level \(\alpha\) varies.}
    \label{fig:alpha_sweep}
\end{figure}

\subsection{Distributional Robustness Results}

Finally, we evaluate ACP-regret across Gaussian, Student-\(t\), and skewed Student-\(t\) reward families. The results are reported in Figure~\ref{fig:distribution_robustness}. Across all three settings, ACP-UCB1 has substantially smaller ACP-regret than mean-based UCB1. This suggests that the metric-specific behavior is not only a Gaussian artifact: when the upper-tail ranking differs from the mean ranking, a mean-based policy can accumulate large ACP-regret, while ACP-UCB1 learns the arm favored by the upper-tail criterion.

\begin{figure}[t]
    \centering
    \includegraphics[width= 1\textwidth]{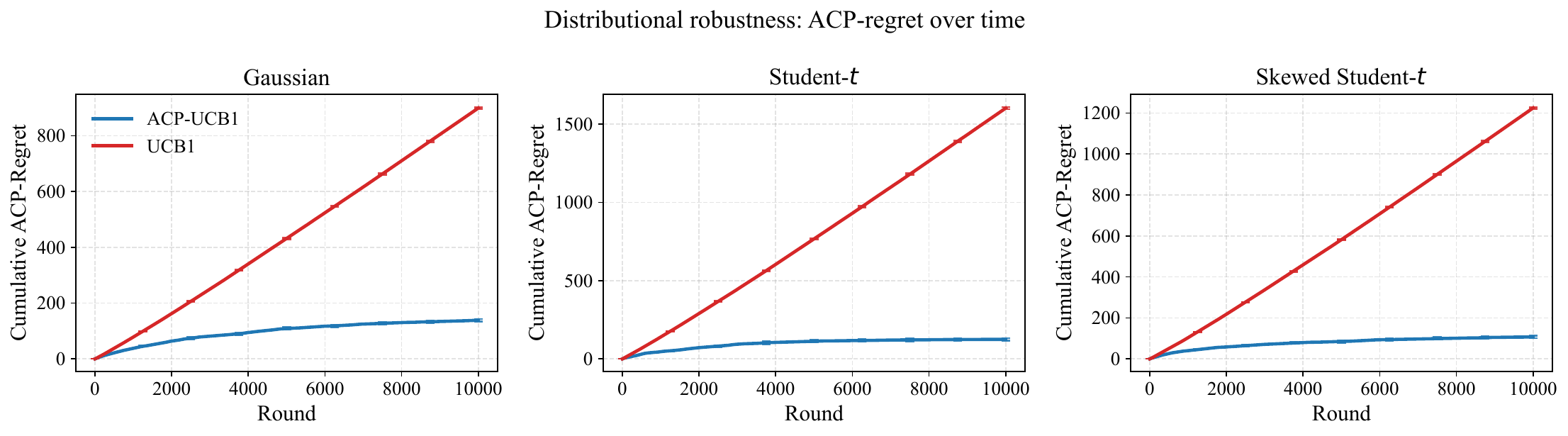}
    \caption{
    Distributional robustness under cumulative ACP-regret.
    }\vspace{-11pt}
    \label{fig:distribution_robustness}
\end{figure}

\section{Conclusion}
\label{sec:conclusion}

This paper introduced ACP-UCB1, a stochastic bandit policy for a conformal
objective induced by a score-calibrated predictive endpoint. The method combines
three ingredients: empirical central reward anchors, a recomputed score-quantile
correction, and a UCB exploration bonus. This construction treats the corrected
upper endpoint as the arm value.

The theoretical analysis shows that, under i.i.d. rewards and local density
regularity, ACP-UCB1 achieves logarithmic pseudo-regret for the conformal
upper-tail objective. The proof replaces empirical-mean concentration in the
standard UCB argument with concentration of empirical reward quantiles and
recomputed score-quantile corrections. The localized adaptive update keeps the
level-induced correction within the exploration scale, allowing the regret
analysis to remain centered on the nominal population endpoint.

The Gaussian comparison and numerical experiments illustrate the objective
dependence of bandit learning. When the mean-optimal arm and the ACP-optimal arm
coincide, the distinction between mean-based and ACP-based policies is reflected
in the logarithmic constants and gaps. When the two objectives select different
arms, a policy optimized for one metric can incur a linear leading term under
the other. The experiments confirm this behavior: ACP-UCB1 reduces regret under
the upper-tail objective, while mean-based UCB remains preferable under the
classical mean objective. Thus ACP-UCB1 is best understood not as a universal
replacement for mean-based bandit algorithms, but as a policy for settings in
which upper predictive performance is the relevant decision criterion.

\begin{ack}
\end{ack}

\bibliographystyle{plainnat}
{\small
\bibliography{references}
}

\newpage 
\appendix

\section{Technical appendices and supplementary material}

\subsection{Proofs for the ACP-UCB1 Regret Theorem}
\label{app:technical}

\subsubsection{Remark in Assumption}
\label{app:proof_normal_local_constants}
Assume that the reward distribution of arm \(j\) is
\(F_j=N(\mu_j,\sigma_j^2)\), with
\(0<\sigma_{\min}\le\sigma_j\le\sigma_{\max}<\infty\) uniformly over
\(j\).  

Let \(z_p:=\Phi^{-1}(p)\),
\[
z_{\min}:=z_{1-\overline\alpha/2},
\qquad
z_{\max}:=z_{1-\underline\alpha/2}.
\]  
If \(r_y>0\) and \(0<r_s<\sigma_{\min}z_{\min}\), then
Assumption~\ref{assum:density} holds with constants independent of \(j\).

One admissible choice is
\[
f_{\min}^{\mathcal N}
=\frac{1}{\sigma_{\max}}\phi\!\left(z_{\max}+\frac{r_y}{\sigma_{\min}}\right),
\qquad
f_{\max}^{\mathcal N}
=\frac{1}{\sigma_{\min}}\phi(0),
\]
\[
g_{\min}^{\mathcal N}
=\frac{2}{\sigma_{\max}}
\phi\!\left(z_{\max}+\frac{r_s}{\sigma_{\min}}\right),
\]
where \(\phi\) and \(\Phi\) denote the standard Gaussian density and CDF.
\begin{proof}
Fix an arm \(j\). Write \(\sigma=\sigma_j\), \(\mu=\mu_j\), and
\(z_p:=\Phi^{-1}(p)\). Set \(a:=z_{1-\alpha/2}\) and, for
\(\beta\in[\underline\alpha,\overline\alpha]\),
\[
z_\beta:=z_{1-\beta/2}.
\]

Since \(\beta\mapsto z_{1-\beta/2}\) is decreasing on
\([\underline\alpha,\overline\alpha]\),
\[
z_\beta\in[z_{\min},z_{\max}],
\qquad
z_{\min}:=z_{1-\overline\alpha/2},
\qquad
z_{\max}:=z_{1-\underline\alpha/2}.
\]
Moreover, \(\overline\alpha<1\) implies \(z_{\min}>0\).

Since \(F_j=N(\mu,\sigma^2)\),
\[
F_j^{-1}(p)=\mu+\sigma z_p.
\]
In particular,
\begin{equation}
q_j^{\rm low}=\mu-\sigma a,
\qquad
q_j^{\rm high}=\mu+\sigma a.
\label{eq:normal_nominal_anchors_appendix}
\end{equation}

For \(z\ge -\sigma a\), the population score CDF satisfies
\begin{align}
G_j(z)
&=F_j(q_j^{\rm high}+z)-F_j(q_j^{\rm low}-z)  \notag\\
&=\Phi\!\left(a+\frac{z}{\sigma}\right)
  -\Phi\!\left(-a-\frac{z}{\sigma}\right) \notag\\
&=2\Phi\!\left(a+\frac{z}{\sigma}\right)-1 .
\label{eq:normal_score_cdf_appendix}
\end{align}

Solving \(G_j(z)=1-\beta\) gives
\begin{equation}
c_j(\beta)=\sigma(z_\beta-a).
\label{eq:normal_score_quantile_appendix}
\end{equation}

Combining \eqref{eq:normal_nominal_anchors_appendix} and
\eqref{eq:normal_score_quantile_appendix} yields
\begin{equation}
q_j^{\rm low}-c_j(\beta)=\mu-\sigma z_\beta,
\qquad
q_j^{\rm high}+c_j(\beta)=\mu+\sigma z_\beta .
\label{eq:normal_beta_endpoints_appendix}
\end{equation}

Equation~\eqref{eq:normal_beta_endpoints_appendix} implies that, for every
\(x\in\mathcal I_j^{\rm rew}\), there exist
\(\varepsilon\in\{-1,+1\}\), \(\beta\in[\underline\alpha,\overline\alpha]\), and
\(|\delta|\le r_y\) such that
\begin{equation}
x=\mu+\varepsilon\sigma z_\beta+\delta .
\label{eq:normal_reward_set_representation}
\end{equation}

It follows that
\begin{equation}
\left|\frac{x-\mu}{\sigma}\right|
\le z_\beta+\frac{|\delta|}{\sigma}
\le z_{\max}+\frac{r_y}{\sigma_{\min}}.
\label{eq:reward_neighborhood_inclusion_appendix}
\end{equation}

Since \(f_j(x)=\sigma^{-1}\phi((x-\mu)/\sigma)\), \(\sigma\le\sigma_{\max}\),
and \(\phi\) is even and decreasing on \([0,\infty)\),
\eqref{eq:reward_neighborhood_inclusion_appendix} yields
\begin{equation}
f_j(x)
\ge
\frac1{\sigma_{\max}}
\phi\!\left(z_{\max}+\frac{r_y}{\sigma_{\min}}\right)
=:f_{\min}^{\mathcal N}.
\label{eq:reward_density_lower_appendix}
\end{equation}

Since \(\phi\) is maximized at zero,
\begin{equation}
f_j(x)
\le \frac1\sigma\phi(0)
\le \frac1{\sigma_{\min}}\phi(0)
=:f_{\max}^{\mathcal N}.
\label{eq:reward_density_upper_appendix}
\end{equation}

This verifies Assumption~\ref{assum:density}(a).

For Assumption~\ref{assum:density}(b), differentiating
\eqref{eq:normal_score_cdf_appendix} on the interior of the score support gives
\begin{equation}
g_j(z)=\frac{2}{\sigma}\phi\!\left(a+\frac{z}{\sigma}\right),
\qquad z>-\sigma a.
\label{eq:normal_score_density_appendix}
\end{equation}

Let \(z\in\mathcal I_j^{\rm score}\). By the definition of
\(\mathcal I_j^{\rm score}\) and \eqref{eq:normal_score_quantile_appendix},
there exist \(\beta\in[\underline\alpha,\overline\alpha]\) and
\(|\delta|\le r_s\) such that
\begin{equation}
z=c_j(\beta)+\delta=\sigma(z_\beta-a)+\delta .
\label{eq:normal_score_set_representation}
\end{equation}
Hence
\[
a+\frac{z}{\sigma}=z_\beta+\frac{\delta}{\sigma}.
\]
Since \(r_s<\sigma_{\min}z_{\min}\),
\begin{equation}
z_\beta+\frac{\delta}{\sigma}
\ge z_{\min}-\frac{r_s}{\sigma_{\min}}>0.
\label{eq:score_support_interior_appendix}
\end{equation}
Thus every point in \(\mathcal I_j^{\rm score}\) lies in the interior of the
score support, where \eqref{eq:normal_score_density_appendix} is valid.
Moreover,
\[
0<a+\frac{z}{\sigma}
\le z_{\max}+\frac{r_s}{\sigma_{\min}}.
\]
Using \eqref{eq:normal_score_density_appendix}, \(\sigma\le\sigma_{\max}\), and
the monotonicity of \(\phi\) on \([0,\infty)\), we obtain
\begin{equation}
g_j(z)
\ge
\frac{2}{\sigma_{\max}}
\phi\!\left(z_{\max}+\frac{r_s}{\sigma_{\min}}\right)
=:g_{\min}^{\mathcal N}.
\label{eq:score_density_lower_appendix}
\end{equation}
Thus Assumption~\ref{assum:density}(b) holds, with constants depending only on
\(\sigma_{\min},\sigma_{\max},\underline\alpha,\overline\alpha,r_y\), and
\(r_s\), and therefore uniform over arms.
\end{proof}

\subsubsection{Auxiliary Lemma}
\label{app:aux_concentration_facts}

\begin{lemma}
\label{lem:local_inverse}
Let \(G\) be continuous and let \(q\) be the unique \(\gamma\)-quantile of
\(G\), so that \(G(q)=\gamma\). Suppose that \(G\) has density at least
\(g_0>0\) on \([q-r,q+r]\). Then, for every \(0<\delta\le r\),
\[
G(q+\delta)\ge \gamma+g_0\delta,
\qquad
G(q-\delta)\le \gamma-g_0\delta .
\]
Consequently, if \(\widehat G\) is a distribution function and
\[
\sup_z|\widehat G(z)-G(z)|\le g_0\delta,
\]
then every lower empirical \(\gamma\)-quantile
\[
\widehat q:=\inf\{z:\widehat G(z)\ge\gamma\}
\]
satisfies
\[
|\widehat q-q|\le\delta .
\]
\end{lemma}

\begin{proof}
The density lower bound gives
\[
G(q+\delta)-G(q)
=\int_q^{q+\delta}g(u)\,du
\ge g_0\delta,
\qquad
G(q)-G(q-\delta)
=\int_{q-\delta}^{q}g(u)\,du
\ge g_0\delta .
\]
Hence
\[
G(q+\delta)\ge \gamma+g_0\delta,
\qquad
G(q-\delta)\le \gamma-g_0\delta .
\]

Let \(\varepsilon=\sup_z|\widehat G(z)-G(z)|\). If
\(\varepsilon<g_0\delta\), then
\[
\widehat G(q+\delta)
\ge G(q+\delta)-\varepsilon
> \gamma,
\]
so \(\widehat q\le q+\delta\). Similarly, we have
\(\widehat q\ge q-\delta\), which yields
\(|\widehat q-q|\le\delta\).
\end{proof}

\begin{lemma}
\label{lem:quantile}
Let \(X_1,\dots,X_m\) be i.i.d. with continuous CDF \(F\) and density at
least \(f_{\min}>0\) on \([q-r_y,q+r_y]\), where \(q:=F^{-1}(\tau)\).
Let \(\hat q_m\) be the empirical \(\tau\)-quantile.  Then for every
\(0<\varepsilon\le r_y\),
\[
\Pp\bigl(|\hat q_m-q|>\varepsilon\bigr)
\le
2\exp(-2m\varepsilon^2 f_{\min}^2).
\]
\end{lemma}

\begin{proof}
Let \(\widehat F_m\) denote the empirical CDF of \(X_1,\ldots,X_m\). By
definition,
\[
\hat q_m=\inf\{z:\widehat F_m(z)\ge\tau\}.
\]
The event \(\{\hat q_m>q+\varepsilon\}\) implies
\(\widehat F_m(q+\varepsilon)<\tau\). Since \(F\) has density at least
\(f_{\min}\) on \([q-r_y,q+r_y]\) and \(0<\varepsilon\le r_y\),
\[
F(q+\varepsilon)-F(q)
=
\int_q^{q+\varepsilon} f(u)\,du
\ge
\varepsilon f_{\min}.
\]
Because \(F(q)=\tau\), we have
\[
F(q+\varepsilon)\ge \tau+\varepsilon f_{\min}.
\]
Thus
\[
\Pp(\hat q_m-q>\varepsilon)
\le
\Pp\!\left(
F(q+\varepsilon)-\widehat F_m(q+\varepsilon)
\ge
\varepsilon f_{\min}
\right).
\]
The random variables \(\ind\{X_r\le q+\varepsilon\}\) are independent Bernoulli
variables with mean \(F(q+\varepsilon)\), by Hoeffding's inequality \citep{hoeffding1963probability}.
\[
\Pp(\hat q_m-q>\varepsilon)
\le
\exp(-2m\varepsilon^2f_{\min}^2).
\]

For the lower tail, the event \(\{\hat q_m<q-\varepsilon\}\) implies
\(\widehat F_m(q-\varepsilon)\ge\tau\). Since
\[
F(q)-F(q-\varepsilon)
=
\int_{q-\varepsilon}^{q}f(u)\,du
\ge
\varepsilon f_{\min},
\]
we have \(F(q-\varepsilon)\le\tau-\varepsilon f_{\min}\). Hence
\[
\Pp(q-\hat q_m>\varepsilon)
\le
\Pp\!\left(
\widehat F_m(q-\varepsilon)-F(q-\varepsilon)
\ge
\varepsilon f_{\min}
\right)
\le
\exp(-2m\varepsilon^2f_{\min}^2).
\]
\end{proof}

The second ingredient transfers this bound to the conformal correction, accounting for the plug-in endpoints used to recompute the scores.

\begin{lemma}
\label{lem:score_quantile}
Under the i.i.d. reward model and Assumption~\ref{assum:density}, for every
\(0<\varepsilon\le\rho\), where \(\rho=\min\{r_s,r_y/2\}\), and every
\(m=s-1\ge1\), define \(\widehat c_{j,s}(\beta)
:=
\widehat Q_{s-1}
\left(
1-\beta;
\widetilde S^{(s)}_{j,1},\ldots,\widetilde S^{(s)}_{j,s-1}
\right),\)
\begin{equation}
\Pp\!\left(
\sup_{\beta\in[\underline\alpha,\overline\alpha]}
\bigl|\widehat c_{j,s}(\beta)-c_j(\beta)\bigr|
\ge\varepsilon
\right)
\le
2\exp\!\left(-\frac{m\varepsilon^2g_{\min}^2}{2}\right)
+
4\exp\!\left(-\frac{m\varepsilon^2f_{\min}^2}{8}\right).
\label{eq:score_quantile_concentration}
\end{equation}
In particular, the bound holds at the random level \(\beta=\alpha_{j,s}\).
\end{lemma}

\begin{proof}
Fix an arm \(j\) and local time \(s\), and set \(m=s-1\). For
\(\beta\in[\underline\alpha,\overline\alpha]\), define the oracle-score
empirical correction
\[
\widehat c^{\rm pop}_{j,s}(\beta)
:=
\widehat Q_m
\left(
1-\beta;
S^{\rm pop}_{j,1},\ldots,S^{\rm pop}_{j,m}
\right),
\]
and let
\[
\widehat G_m^{\rm pop}(z)
:=
\frac1m\sum_{r=1}^m
\ind\{S^{\rm pop}_{j,r}\le z\}
\]
be the empirical CDF of the oracle scores. By the
Dvoretzky--Kiefer--Wolfowitz inequality with Massart's constant
\citep{dvoretzky1956asymptotic,massart1990tight},
\begin{equation}
\Pp\!\left(
\sup_z|\widehat G_m^{\rm pop}(z)-G_j(z)|>u
\right)
\le
2\exp(-2mu^2).
\label{eq:dkw_oracle_score_app}
\end{equation}
Taking \(u=g_{\min}\varepsilon/2\), and using
Lemma~\ref{lem:local_inverse}, gives
\begin{equation}
\Pp\!\left(
\sup_{\beta\in[\underline\alpha,\overline\alpha]}
\left|
\widehat c^{\rm pop}_{j,s}(\beta)-c_j(\beta)
\right|
>
\varepsilon/2
\right)
\le
2\exp\!\left(-\frac{m\varepsilon^2g_{\min}^2}{2}\right).
\label{eq:oracle_correction_bound_app}
\end{equation}

Next define
\[
E_{j,s}
:=
\left|\hat q^{\rm low}_{j,s-1}-q_j^{\rm low}\right|
+
\left|\hat q^{\rm high}_{j,s-1}-q_j^{\rm high}\right|.
\]
By \eqref{eq:score_plugin_bound},
\[
\max_{1\le r\le m}
\left|
\widetilde S^{(s)}_{j,r}-S^{\rm pop}_{j,r}
\right|
\le
E_{j,s}.
\]

Thus we have
\[
\sup_{\beta\in[\underline\alpha,\overline\alpha]}
\left|
\widehat c_{j,s}(\beta)-\widehat c^{\rm pop}_{j,s}(\beta)
\right|
\le
E_{j,s}.
\]
Consequently,
\[
\left\{
\sup_{\beta\in[\underline\alpha,\overline\alpha]}
\left|
\widehat c_{j,s}(\beta)-c_j(\beta)
\right|
>
\varepsilon
\right\}
\subseteq
\mathcal E_{1}\cup\mathcal E_{2},
\]
where
\[
\mathcal E_{1}
:=
\left\{
\sup_{\beta\in[\underline\alpha,\overline\alpha]}
\left|
\widehat c^{\rm pop}_{j,s}(\beta)-c_j(\beta)
\right|
>
\varepsilon/2
\right\},
\qquad
\mathcal E_{2}:=\{E_{j,s}>\varepsilon/2\}.
\]

The event \(\mathcal E_2\) implies that one of the two reward-quantile errors
exceeds \(\varepsilon/4\). Since
\(\varepsilon\le\rho\le r_y/2\), Lemma~\ref{lem:quantile} gives
\begin{equation}
\Pp(\mathcal E_2)
\le
4\exp\!\left(-\frac{m\varepsilon^2f_{\min}^2}{8}\right).
\label{eq:plugin_anchor_bound_app}
\end{equation}
Combining \eqref{eq:oracle_correction_bound_app} and
\eqref{eq:plugin_anchor_bound_app} yields
\[
\Pp\!\left(
\sup_{\beta\in[\underline\alpha,\overline\alpha]}
\left|
\widehat c_{j,s}(\beta)-c_j(\beta)
\right|
>
\varepsilon
\right)
\le
2\exp\!\left(-\frac{m\varepsilon^2g_{\min}^2}{2}\right)
+
4\exp\!\left(-\frac{m\varepsilon^2f_{\min}^2}{8}\right).
\]
\end{proof}

\subsubsection{Proof of Theorem 1}
\label{app:proof_main}

\begin{proof}
Fix a suboptimal arm \(j\ne A\), and set
\[
\ell_j
:=
\left\lceil
\frac{4b^2\log n}{(\Delta_j^{\rm ACP})^2}
\right\rceil
\vee s_0 .
\]
For a post-initialization round \(t\), let
\[
s:=T_j(t-1),
\qquad
s_*:=T_A(t-1).
\]
If ACP-UCB1 selects arm \(j\) at round \(t\) and \(s\ge \ell_j\), then at least
one of the following events occurs:
\[
E_1(s_*,t)
:=
\left\{
U_{A,s_*+1}
\le
u_A^\star
-
b\sqrt{\frac{\log t}{s_*}}
\right\},
\]
\[
E_2(s,t)
:=
\left\{
U_{j,s+1}
\ge
u_j^\star
+
b\sqrt{\frac{\log t}{s}}
\right\},
\]
or
\[
E_3(s,t)
:=
\left\{
\Delta_j^{\rm ACP}
<
2b\sqrt{\frac{\log t}{s}}
\right\}.
\]
Since \(s\ge\ell_j\) and \(t\le n\), the event \(E_3(s,t)\) is
excluded. Therefore,
\begin{equation}
\E[T_j(n)]
\le
\ell_j
+
\sum_{t=1}^n
\sum_{s_*=s_0}^{t-1}
\sum_{s=\ell_j}^{t-1}
\left\{
\Pp(E_1(s_*,t))+\Pp(E_2(s,t))
\right\}.
\label{eq:auer_triple_sum_acp}
\end{equation}

Set
\[
a_{s,t}:=b\sqrt{\frac{\log t}{s}}.
\]

Using \(c_j(\alpha)=0\) and \(u_j^\star=q_j^{\rm high}\), decompose
\[
U_{j,s+1}-u_j^\star
=
R_{j,s}+Q_{j,s}+A_{j,s},
\]
where
\[
R_{j,s}
:=
\hat q^{\rm high}_{j,s}-q_j^{\rm high},
\]
\[
Q_{j,s}
:=
\widehat c_{j,s+1}(\alpha_{j,s+1})
-
c_j(\alpha_{j,s+1}),
\]
and
\[
A_{j,s}
:=
c_j(\alpha_{j,s+1})-c_j(\alpha).
\]

The localized projection gives
\[
|\alpha_{j,s+1}-\alpha|
\le
\lambda\sqrt{\frac{\log s}{s}},
\qquad s\ge2.
\]
By Lemma~\ref{lem:local_inverse} applied to \(G_j\),
\[
|c_j(\beta)-c_j(\alpha)|\le |\beta-\alpha|/g_{\min},
\qquad
\beta\in[\underline\alpha,\overline\alpha].
\] 

Hence
\[
|A_{j,s}|
\le
\frac{|\alpha_{j,s+1}-\alpha|}{g_{\min}}
\le
\frac{\lambda}{g_{\min}}
\sqrt{\frac{\log s}{s}}
\le
\frac{b}{4}
\sqrt{\frac{\log t}{s}}
=
\frac{a_{s,t}}{4}.
\]

Thus
\[
E_2(s,t)
\subseteq
\{|R_{j,s}|\ge a_{s,t}/4\}
\cup
\{|Q_{j,s}|\ge a_{s,t}/4\}.
\]
must hold. Lemma~\ref{lem:quantile}, applied to the empirical upper reward
quantile, gives
\[
\Pp(|R_{j,s}|\ge a_{s,t}/4)
\le
2\exp\!\left(
-\frac{b^2f_{\min}^2}{8}\log t
\right)
\le
2t^{-4},
\]
where the last inequality follows from \(b\ge b_0=64/h_{\min}\). Lemma
\ref{lem:score_quantile}, applied at local time \(s+1\), gives
\[
\Pp(|Q_{j,s}|\ge a_{s,t}/4)
\le
2\exp\!\left(
-\frac{b^2g_{\min}^2}{32}\log t
\right)
+
4\exp\!\left(
-\frac{b^2f_{\min}^2}{128}\log t
\right)
\le
6t^{-4},
\]
again by \(b\ge64/h_{\min}\). Therefore,
\begin{equation}
\Pp(E_2(s,t))\le C_Et^{-4}
\label{eq:E2_bound_closed}
\end{equation}
for a finite numerical constant \(C_E\). The same argument gives
\[
\Pp(E_1(s_*,t))\le C_Et^{-4}.
\]

Substituting these bounds into \eqref{eq:auer_triple_sum_acp},
\[
\sum_{t=1}^{\infty}
\sum_{s_*=s_0}^{t-1}
\sum_{s=\ell_j}^{t-1}
t^{-4}
\le
\sum_{t=1}^{\infty}t^2t^{-4}
<\infty .
\]
The triple sum is absorbed into a finite constant \(B_j\), independent of \(n\).
Indeed, \(r_s\ge\Delta_j^{\rm ACP}/4\) and \(r_y/2\ge\Delta_j^{\rm ACP}/8\), so
\(\rho\ge\Delta_j^{\rm ACP}/8\).
Therefore
\[
s_0
=
\left\lceil
\frac{b^2\log n}{16\rho^2}
\right\rceil\vee2
\le
\frac{4b^2\log n}{(\Delta_j^{\rm ACP})^2}+O(1).
\]
Since
\[
\ell_j
=
\left\lceil
\frac{4b^2\log n}{(\Delta_j^{\rm ACP})^2}
\right\rceil
\vee s_0,
\]
we obtain
\[
\E[T_j(n)]
\le
\frac{4b^2\log n}{\Delta_j^2}+B_j.
\]
\end{proof}

\subsubsection{Proofs for the Gaussian Comparison}
\label{app:normal_comparison_proofs}

\begin{proof}[Proof of Proposition~\ref{prop:objective_transfer}]
Fix \(\kappa\in[0,a]\). Expanding the regret by arms gives
\[
R_n^\kappa(\pi_g)
=
\sum_h D_h(\kappa)\E_{\pi_g}[T_h(n)].
\]
Since \(\sum_h\E_{\pi_g}[T_h(n)]=n\), we may write
\[
D_g(\kappa)n
=
D_g(\kappa)\E_{\pi_g}[T_g(n)]
+
D_g(\kappa)\sum_{h\ne g}\E_{\pi_g}[T_h(n)].
\]
Subtracting,
\[
R_n^\kappa(\pi_g)-D_g(\kappa)n
=
\sum_{h\ne g}
\bigl(D_h(\kappa)-D_g(\kappa)\bigr)
\E_{\pi_g}[T_h(n)].
\]
Let
\[
D_{\max}:=
\max_h\sup_{\kappa\in[0,a]}D_h(\kappa).
\]
For a fixed finite Gaussian instance, \(D_{\max}<\infty\). Hence
\[
\sup_{\kappa\in[0,a]}
\left|
R_n^\kappa(\pi_g)-D_g(\kappa)n
\right|
\le
2D_{\max}
\sum_{h\ne g}\E_{\pi_g}[T_h(n)].
\]
Taking \(g=A\) and \(g=M\) gives
\eqref{eq:acp_kappa_leading} and \eqref{eq:mean_kappa_leading}.
\end{proof}

\begin{proof}[Derivation of \eqref{eq:two_arm_D_functions}--\eqref{eq:two_arm_suprema}]
Under the two-arm envelope condition,
\[
v_\star(\kappa)=\max\{v_A(\kappa),v_M(\kappa)\}.
\]
Moreover,
\[
v_M(\kappa)-v_A(\kappa)
=
(\mu_M-\mu_A)+\kappa(\sigma_M-\sigma_A)
=
\delta_\mu-\kappa\delta_\sigma .
\]
Therefore,
\[
D_A(\kappa)
=
(v_M(\kappa)-v_A(\kappa))_+
=
(\delta_\mu-\kappa\delta_\sigma)_+,
\]
and
\[
D_M(\kappa)
=
(v_A(\kappa)-v_M(\kappa))_+
=
(\kappa\delta_\sigma-\delta_\mu)_+.
\]
This proves \eqref{eq:two_arm_D_functions}.

Since \(\kappa_0=\delta_\mu/\delta_\sigma\in(0,a)\), the maximum of \(D_A\) on
\([0,a]\) is attained at \(\kappa=0\), and the maximum of \(D_M\) is attained at
\(\kappa=a\). Hence
\[
\sup_{\kappa\in[0,a]}D_A(\kappa)=\delta_\mu,
\qquad
\sup_{\kappa\in[0,a]}D_M(\kappa)=a\delta_\sigma-\delta_\mu,
\]
which proves \eqref{eq:two_arm_suprema}. It follows immediately that
\[
\sup_{\kappa\in[0,a]}D_A(\kappa)
<
\sup_{\kappa\in[0,a]}D_M(\kappa)
\quad\Longleftrightarrow\quad
2\delta_\mu<a\delta_\sigma .
\]

More generally,
\[
\frac1a\int_0^aD_A(\kappa)\,d\kappa
=
\frac1a\int_0^{\kappa_0}
(\delta_\mu-\kappa\delta_\sigma)\,d\kappa
=
\frac{\delta_\mu^2}{2a\delta_\sigma},
\]
and
\[
\frac1a\int_0^aD_M(\kappa)\,d\kappa
=
\frac1a\int_{\kappa_0}^{a}
(\kappa\delta_\sigma-\delta_\mu)\,d\kappa
=
\frac{(a\delta_\sigma-\delta_\mu)^2}{2a\delta_\sigma}.
\]
\end{proof}

\subsection{Details on Numerical Experiments in Section~\ref{sec:experiments}}
\label{app:experiment_details}

We report the experimental details for Section~\ref{sec:experiments}, including implementation details and data generation.

\paragraph{Algorithms and implementation}
All experiments are run with horizon \(n=10{,}000\). Unless otherwise stated, the nominal level is \(\alpha=0.1\), so the ACP target is the \(0.95\)-quantile of the reward distribution. ACP-UCB1 uses exploration coefficient \(b=1\), adaptive step size \(\eta=0.01\), two initial round-robin pulls per arm, projection interval \([10^{-4},0.5]\), and localization parameter \(\lambda=1\). The mean-based UCB1 benchmark uses the index
\[
\hat\mu_j+\sqrt{\frac{2\log t}{T_j(t-1)}}.
\]
For cumulative regret curves, we average over 60 Monte Carlo replications. For the nominal-level sweep and distributional robustness experiments, we average over 30 Monte Carlo replications.

\paragraph{Data generation}
For the Gaussian objective-mismatch experiment, the three arms are
\[
N(0.10,0.05^2),\qquad N(0.00,0.15^2),\qquad N(0.04,0.08^2).
\]
At \(\alpha=0.1\), their ACP values are approximately \(0.182\), \(0.247\), and \(0.172\), respectively. Thus the first arm is mean-optimal, while the second arm is ACP-optimal.

For the nominal-level sweep, we use the same Gaussian instance and vary
\[
    \alpha\in\{0.30,0.25,0.20,0.15,0.10,0.05\}.
\]
For the distributional robustness experiment, all three distribution families
use the same location-scale pattern
\[
    (\ell_1,\ell_2,\ell_3)=(0.08,0.00,0.03),\qquad
    (s_1,s_2,s_3)=(0.05,0.16,0.08).
\]
The Gaussian case uses \(Y_j\sim N(\ell_j,s_j^2)\). The Student-\(t\) case uses
\(Y_j=\ell_j+s_j T_{\nu_j}\), where \(T_{\nu_j}\) is a standard Student-\(t\)
random variable and
\[
    (\nu_1,\nu_2,\nu_3)=(3,3,5).
\]
The skewed Student-\(t\) case uses
\(Y_j=\ell_j+s_j Z_{\nu_j,\lambda_j}\), where \(Z_{\nu_j,\lambda_j}\) denotes
the standardized skewed Student-\(t\) variable used in the simulation code, with
\[
    (\nu_1,\nu_2,\nu_3)=(5,5,5),\qquad
    (\lambda_1,\lambda_2,\lambda_3)=(-0.2,0.6,0.3).
\]

The full reference implementation,
including the configuration files and random seeds used to generate every figure in this section, is provided as an anonymized release at \url{https://anonymous.4open.science/r/Conformal-Style-Quantile-Analyses-for-Stochastic-Bandits-0703/}. And all experiments were run on a workstation with an AMD Ryzen Threadripper 7960X CPU (24 cores, 48 threads) and 384 GB of RAM; no GPU is required. The full reproducibility script completes in approximately 7 minutes.

\end{document}